\documentclass[11pt]{article}

\usepackage{latexsym,mathrsfs}
\usepackage{amsmath,amssymb}
\usepackage{amsthm,enumerate,verbatim}
\usepackage{amsfonts}
\usepackage{graphicx}
\usepackage{algorithm}
\usepackage{algorithmic}
\usepackage{url}
\usepackage{mathtools}
\usepackage{color}
\def\blue{\color{blue}}

\renewcommand{\algorithmicrequire}{\textbf{Input:}}
\renewcommand{\algorithmicensure}{\textbf{Output:}}

\newtheorem{lemma}{Lemma}

\newtheorem{theorem}{Theorem}
\newtheorem{corollary}{Corollary}
\newtheorem{definition}{Definition}
\newtheorem*{definition*}{Definition}

\newtheorem{assumption}{Assumption}
\newtheorem{remark}{Remark}

\DeclareMathOperator{\argmax}{argmax}

\DeclareMathOperator{\rank}{rank}

\DeclareMathOperator{\tr}{tr}

\title{Enhancing Pure-Pixel Identification Performance via Preconditioning}
\date{}

\author{Nicolas Gillis \\ 
Department of Mathematics and Operational Research \\
Facult\'e Polytechnique, Universit\'e de Mons \\
Rue de Houdain 9, 7000 Mons, Belgium\\
 nicolas.gillis@umons.ac.be
 \and
  Wing-Kin Ma  \\ 
 Department of Electronic Engineering   \\
The Chinese University of Hong Kong   \\
wkma@ee.cuhk.edu.hk
}

\begin{document}

\maketitle

\begin{abstract}
In this paper, we analyze different preconditionings designed to enhance robustness of pure-pixel search algorithms, which are used for blind hyperspectral unmixing and which are equivalent to near-separable nonnegative matrix factorization algorithms. Our analysis focuses on the successive projection algorithm (SPA), a simple, efficient and provably robust algorithm in the pure-pixel algorithm class. Recently, a provably robust preconditioning was proposed by Gillis and Vavasis (arXiv:1310.2273) which requires the resolution of a semidefinite program (SDP)  to find a data points-enclosing minimum volume ellipsoid. Since solving the SDP  in high precisions can be time consuming, we generalize the robustness analysis to approximate solutions of the SDP, that is, solutions whose objective function values are some multiplicative  factors away from the optimal value. It is shown that a high accuracy solution is not crucial for robustness,   which paves the way for faster preconditionings (e.g., based on first-order optimization methods). This first contribution also allows us to provide a robustness analysis for two other preconditionings. The first one is pre-whitening, which can be interpreted as an optimal solution of the same SDP with additional constraints. We analyze robustness of pre-whitening which allows us to characterize situations in which it performs competitively with the SDP-based preconditioning. The second one is based on SPA itself and can be interpreted as an optimal solution of a relaxation of the SDP. It is extremely fast while competing with the SDP-based preconditioning on several synthetic data sets.
\end{abstract}


\section{Introduction}

Given a hyperspectral image, blind hyperspectral unmixing (blind HU) aims at recovering the spectral signatures of the constitutive materials present in the image, called endmembers, along with their abundances in each pixel. Under the \emph{linear mixing model}, the spectral signature of a pixel is equal to a linear combination of the spectral signatures of the endmembers where the weights correspond to the abundances.
More formally, letting $X \in \mathbb{R}^{m \times n}_+$ represent a hyperspectral image with $m$ wavelengths and $n$ pixels, we have, in the noiseless case,
\[
X(:,j) = \sum_{k=1}^r W(:,k) H(k,j) \quad \text{ for all $j$},
\]
where $X(:,j)$ is the spectral signature of the $j$th pixel, \mbox{$W(:,k)$} the spectral signature of the $k$th endmember, and $H(k,j)$ is the abundance of the $k$th endmember in the $j$th pixel so that $H \geq 0$ and
$||H(:,j)||_1 = \sum_{i=1}^r |H(i,j)|  = 1$ for all $j$ (abundance sum-to-one constraint).
Note that blind HU is equivalent to nonnegative matrix factorization (NMF) which aims at finding the best possible factorization of a nonnegative matrix $X \approx WH$ where $W$ and $H$ are nonnegative matrices.

In blind HU, the so-called pure-pixel assumption plays a significant role. It is defined as follows. 
If for each endmember there exists a pixel containing only that endmember, that is, if for all $1 \leq k \leq r$ there exists $j$ such that $M(:,j)=W(:,k)$, then the pure-pixel assumption holds. In that case, the matrix $X$ has the following form
\begin{align} \nonumber 
X & = W [I_r, H'] \Pi, \quad \text{ with } H' \geq 0,  ||H'(:,j)||_1 = 1 \, \forall j,
\end{align}
and $\Pi$
 being
a permutation.
This implies that the columns of $X$ are convex combinations of the columns of $W$,
 and
hence blind HU under the linear mixing model and the pure-pixel assumption reduces to identifying the vertices of the convex hull of the columns of $X$;
see, e.g., \cite{BP12, Ma14} and the references therein. 
This problem is known to be efficiently solvable \cite{AGKM11}. 
In the presence of noise, the problem becomes more difficult and several provably robust algorithms have been proposed recently; for example the successive projection algorithm to be described in Section~\ref{spasec}.
Note that, in the NMF literature, the pure-pixel assumption is referred to as the separability assumption \cite{AGKM11} and NMF under the separability assumption in the presence of noise is referred to as near-separable NMF; see, e.g., \cite{GL14} and the references therein.
Therefore, in this paper, we will assume that the matrix corresponding to the hyperspectral image has the following form (in the noiseless case):
\begin{assumption}[Separable Matrix] \label{ass1}
The matrix $X$ is separable if $X = W  H \in \mathbb{R}^{m \times n}$ where
\mbox{$W \in \mathbb{R}^{m \times r}$},
$H = [I_r, H']\Pi \in \mathbb{R}^{r \times n}_+$ with the sum of the entries of each column of $H'$ being at most one, that is, $||H(:,j)||_1 \leq 1$ for all $j$, and $\Pi$ is a permutation.
\end{assumption}
Note that we have relaxed the assumption $||H(:,j)||_1 = 1$ for all $j$ to $||H(:,j)||_1 \leq 1$ for all $j$; this allows for example different illumination conditions among the pixels in the image.

\subsection{Successive Projection Algorithm} \label{spasec}

The successive projection algorithm (SPA) is a simple but fast and robust pure-pixel search algorithm;
see Alg.~\ref{spa}.
\renewcommand{\thealgorithm}{SPA}
\algsetup{indent=2em}
\begin{algorithm}[ht!]
\caption{-- Successive Projection Algorithm \cite{MC01} \label{spa}}
\begin{algorithmic}[1]
\REQUIRE Matrix $\tilde{X} = X + N$ with $X$ satisfying Assumption~\ref{ass1}, rank $r$.
\ENSURE Set of $r$ indices $\mathcal{K}$ such that $\tilde{X}(:,\mathcal{K}) \approx W$.
    \medskip

\STATE Let $R = \tilde{X}$, $\mathcal{K} = \{\}$, $k = 1$.
\WHILE {$k \leq r$ and $R \neq 0$}
\STATE $p = \argmax_j ||R_{:j}||_2$.
\STATE $R = \left(I-\frac{{R_{:p}} R_{:p}^T}{||{R_{:p}}||_2^2}\right)R$. \vspace{0.1cm}
\STATE $\mathcal{K} = \mathcal{K} \cup \{p\}$.
\STATE $k = k + 1$.
\ENDWHILE
\end{algorithmic}
\end{algorithm}
At each step of the algorithm, the column of the input matrix $\tilde{X}$ with maximum $\ell_2$ norm  is selected, and then  $\tilde{X}$ is updated by projecting each column onto the orthogonal complement of the columns selected so far. SPA is extremely fast as it can be implemented in  $2mnr + \mathcal{O}(mr^2)$ operations \cite{GV12}. 
SPA was first introduced in~\cite{MC01},
and is closely related to other algorithms such as automatic target generation process (ATGP), successive simplex volume maximization (SVMAX) and vertex component analysis (VCA); see the discussion in~\cite{Ma14}.
What makes SPA distinguishingly interesting is that it is provably robust against noise~\cite{GV12}:
\begin{theorem}[\cite{GV12}, Th.~3] \label{th1}
Let $\tilde{X} = X + N$ where $X$ satisfies Assumption~\ref{ass1},
 $W$ has full column rank and $N$ is noise with \mbox{$\max_j ||N(:,j)||_2 \leq \epsilon$}.
If $\epsilon \leq \mathcal{O} \left( \,  \frac{  \sigma_{\min}(W)  }{\sqrt{r} \kappa^2(W)} \right)$, then \ref{spa} identifies the columns of $W$ up to error $\mathcal{O} \left( \epsilon \, \kappa^2(W) \right)$, that is, the index set $\mathcal{K}$ identified by \ref{spa} satisfies
\[
\max_{1 \leq j \leq r} \min_{k \in \mathcal{K}} \left\|W(:,j) - \tilde{X}(:,k)\right\|_2 \leq \mathcal{O} \left( \epsilon \, \kappa^2(W) \right),
\]
where $\kappa(W) = \frac{\sigma_{\max}(W)}{\sigma_{\min}(W)}$ is the condition number of $W$, and $||x||_2 = \sqrt{\sum_{i=1}^n x_i^2}$ for $x \in \mathbb{R}^n$.
\end{theorem}


\subsection{Preconditioning} \label{sdpsec}

If a matrix ${X}$ satisfying Assumption~\ref{ass1} is premultiplied by a matrix $Q$,
it still satisfies Assumption~\ref{ass1} where $W$ is replaced with $QW$. Since pure-pixel search algorithms are sensitive to the conditioning of matrix $W$, it would be beneficial to find a matrix $Q$
 that reduces
the conditioning of~$W$.
In particular, the robustness result of SPA (Th.~\ref{th1}) can be adapted when the input matrix is premultiplied by a matrix~$Q$:
\begin{corollary} \label{cor1}
Let $\tilde{X} = X + N$ where $X$ satisfies Assumption~\ref{ass1},
 $W$ has full column rank and $N$ is noise with \mbox{$\max_j ||N(:,j)||_2 \leq \epsilon$};
and let $Q \in \mathbb{R}^{p \times m}$ ($p \geq r$).
If $QW$ has full column rank, and
\[
\epsilon \leq \mathcal{O} \left( \frac{\sigma_{\min}(W)}{\sqrt{r} \kappa^3(QW)} \right) ,
\]
then SPA applied on matrix $Q\tilde{X}$ identifies indices corresponding to the columns of $W$
up to error  $\mathcal{O} \left(  \epsilon \,  \kappa(W) \, \kappa(QW)^3  \right)$.
\end{corollary}
\begin{proof}
This result follows directly from \cite[Cor.~1 ]{GV13}.
In fact, in \cite[Cor.~1 ]{GV13}, the result is proved for
\[
\epsilon \leq \mathcal{O} \left( \frac{\sigma_{\min}(QW)}{\sqrt{r} \sigma_{\max}(Q) \kappa^2(QW)} \right)
\]
with error up to $\mathcal{O} \left(  \epsilon \,  \kappa(Q) \kappa(QW)^2  \right)$. Since,
\[
 \frac{\sigma_{\min}(QW)}{\sigma_{\max}(Q)}
\geq \frac{\sigma_{\min}(QW)}{ \sigma_{\max}(QW) \sigma_{\max}(W^{-1})} = \frac{\sigma_{\min}(W)}{\kappa(QW)} ,
\]
and
\[
 \kappa(Q) = \kappa(QWW^{-1}) \leq \kappa(QW)  \kappa(W^{-1}) = \kappa(QW)  \kappa(W),
\]
the proof is complete.
\end{proof}
Note that Corollary~\ref{cor1} does not simply
 amount to replacing
 $W$ by $QW$ in Theorem~\ref{th1} because the noise $N$ is also premultiplied by $Q$.
Note also that, in view of Theorem~\ref{th1}~, preconditioning is beneficial for any $Q$ such that $\kappa(QW)^3 \leq \kappa(W)$.

\subsubsection{SDP-based Preconditioning}

Assume that $m=r$ (the problem can be reduced to this case using noise filtering; see Section~\ref{prewsec}). An optimal preconditioning would be $Q = W^{-1}$ so that $QW = I_r$ would be perfectly conditioned, that is, $\kappa(QW) = 1$.
In particular, applying Corollary~\ref{cor1} with $Q = W^{-1}$ gives the following result:
if $\epsilon \leq \mathcal{O} \left( \frac{\sigma_{\min}(W)}{\sqrt{r} } \right)$,
then SPA applied on matrix $Q\tilde{X}$ identifies indices corresponding to the columns of $W$ up to error  $\mathcal{O} \left(  \epsilon \kappa(W) \right)$. This is a significant improvement compared to Theorem~\ref{th1}, especially for the upper bound on the noise level: the term $\kappa(W)^2$ disappears from the denominator. For hyperspectral images, $\kappa(W)$ can be rather large as spectral signatures often share similar patterns. Note that the bound on the noise level is essentially optimal for SPA since $\epsilon \geq \Omega(\sigma_{\min}(W))$ would allow the noise to make the matrix $W$ rank deficient \cite{GV13}.

Of course, $W^{-1}$ is unknown otherwise the problem would  be 
solved. 
However, it turns out that it is possible to compute $W^{-1}$ approximately (up to orthogonal transformations, which do not influence the conditioning) even in the presence of noise using the minimum volume ellipsoid centered at the origin containing all columns of $\tilde{X}$ \cite{GV13}. An ellipsoid $\mathcal{E}$ centered at the origin in $\mathbb{R}^r$ is described via a positive definite matrix $A \in \mathbb{S}^r_{++}$: $\mathcal{E} = \{  x \in \mathbb{R}^r  |  x^T A x \leq 1   \}$.
The volume of $\mathcal{E}$ is equal to $\det(A)^{-1/2}$ times the volume of the unit ball in dimension $r$. Therefore, given a matrix $\tilde{X} \in \mathbb{R}^{r \times n}$ of rank $r$, we can formulate the minimum volume ellipsoid centered at the origin and containing the columns $\tilde{x_j}$ $\forall j$ of matrix $\tilde{X}$ as follows
\begin{align}
A^*   \quad  =  \quad    
\argmax_{A \in \mathbb{S}^r_+}  & \; \det(A)
  \quad  \text{ such that  } \quad
\tilde{x_j}^T A \tilde{x_j} \leq 1 \; \forall j .  \label{SDPp}
\end{align} 
This problem is SDP representable \cite[p.222]{BV04}.
It was shown in~\cite{GV13} that
(i) in the noiseless case (that is, $N = 0$), the optimal solution $A^*$ of \eqref{SDPp} is given by $(WW^T)^{-1}$
 and
hence factoring $A^*$ allows to recover $W^{-1}$
(up to orthogonal transformations); and that
(ii) in the noisy case, the optimal solution of \eqref{SDPp} is close to $(WW^T)^{-1}$
 and
hence leads to a good preconditioning for SPA. More precisely, the following robustness result was proved:
\begin{theorem}[\cite{GV13}, Th.~3] \label{th2}
Let $\tilde{X} = X + N$ where $X$ satisfies Assumption~\ref{ass1} with $m=r$,
 $W$ has full column rank and $N$ is noise with \mbox{$\max_j ||N(:,j)||_2 \leq \epsilon$}.
If $\epsilon \leq \mathcal{O} \left( \,  \frac{  \sigma_{\min}(W)  }{r \sqrt{r}} \right)$,
then SDP-based preconditioned SPA identifies a subset $\mathcal{K}$ so that \mbox{$\tilde{X}(:,\mathcal{K})$} approximates the columns of $W$ up to error $\mathcal{O} \left( \epsilon \, \kappa(W) \right)$.
\end{theorem}

In case $m > r$, it was proposed to first replace the data points by their projections onto the $r$-dimensional linear subspace obtained with the SVD (that is, use a linear dimensionality reduction technique for noise filtering; see also Section~\ref{prewsec});
see Alg.~\ref{sdpprec}.
\renewcommand{\thealgorithm}{SDP-Prec}
\algsetup{indent=2em}
\begin{algorithm}[ht!]
\caption{-- SDP-based Preconditioning \label{sdpprec} \cite{GV13}}
\begin{algorithmic}[1]
\REQUIRE Matrix $\tilde{X} = X+N$ with $X$ satisfying Assumption~\ref{ass1}, rank $r$.
\ENSURE Preconditioner $Q$.
    \medskip

\STATE $[U_r,\Sigma_r,V_r]$ = rank-$r$ truncated SVD($\tilde{X}$).

\STATE Let $\tilde{X} \leftarrow \Sigma_r V_r^T$ and solve \eqref{SDPp} to get $A^*$.

\STATE Factorize $A^* = P^T P$ (e.g., Cholesky decomposition).

\STATE $Q = P U_r^T$.

\end{algorithmic}
\end{algorithm}

Alg.~\ref{sdpprec} first requires  the truncated SVD which can be computed in $\mathcal{O}(mnr)$ operations.
It then requires the solution of the SDP with $\mathcal{O}(r^2)$ variables and $n$ constraints, which
 takes
$\mathcal{O}(r^6 + n^3)$ operations per iteration
 to compute
if
standard interior point methods
 are used.
However, effective active set methods can be used to solve large-scale problems (see \cite{GV13}):
in fact, one can keep only $\frac{r(r+1)}{2}$ constraints from \eqref{SDPp} to obtain an equivalent problem \cite{john}.\\

Note that Mizutani~\cite{M13b} solves the same SDP, but for another purpose, namely to preprocess the input matrix by removing the columns which are not on the boundary of the minimum volume ellipsoid.

\subsection{Motivation and Contribution of the Paper}

The SDP-based preconditioning described in Section~\ref{sdpsec} is
appealing
in the sense that it builds an approximation of $W^{-1}$
 whose
error
 is provably
bounded by the noise level.
This is a somewhat ideal solution but can be computationally expensive to obtain since it requires the resolution of an SDP.
 Hence, a natural move is to consider computationally cheaper preconditioning alternatives.


The focus of this paper is on a \emph{theoretical analysis of the robustness to noise of several preconditionings}.
The contribution of this paper is
threefold{\blue :}
\begin{enumerate}

\item In Section~\ref{approxsdpprec}, we analyze robustness of preconditionings obtained using approximate solutions of~\eqref{SDPp}
and prove the following (see Theorem~\ref{mainth}):
\begin{quote}
Let $\tilde{X} = X + N$ where $X$ satisfies Assumption~\ref{ass1} with $m=r$,
$W$  has full rank,
and $N$  is  the noise
 and satisfies
\mbox{$\max_j ||N(:,j)||_2 \leq \epsilon$}.
Let also $A = QQ^T$ be a feasible solution of~\eqref{SDPp}
 whose objective function value is some multiplicative factor away from the optimal value; that is,
$\det(A) \geq \alpha \det(A^*)$ for some $0 < \alpha \leq 1,$
 where $A^*$ is an optimal solution to \eqref{SDPp}.
If $\epsilon \leq  \mathcal{O} \left( \min\left( \frac{1}{r}, \alpha^{3/2}  \right) \frac{\sigma_{\min}(W)}{\sqrt{r}} \right)$,
then SPA applied on matrix $Q\tilde{X}$ identifies indices corresponding to the columns of $W$ up to error  $\mathcal{O} \Big( \epsilon \kappa(W) \alpha^{-3/2} \Big)$.
\end{quote}
 The above stated result suggests that
any good approximate solution of \eqref{SDPp} provides a reasonable preconditioning.
This gives
a theoretical motivation
 for developing
less accurate but faster solvers for \eqref{SDPp}; for example,
 one
could use the proximal point algorithm proposed in \cite{YST13}.
 We should mention that this paper focuses on theoretical analysis only, and developing fast solvers is a different subject and is considered a future research topic.

\item In Section~\ref{prewsec}, we analyze robustness of pre-whitening, a standard preconditioning technique in blind source separation.
 We
try to understand under which conditions pre-whitening can be
 as good as the SDP-based preconditioning;
in fact,  pre-whitening was shown to perform very similarly as the SDP-based preconditioning on some synthetic data sets \cite{GV13}.
Pre-whitening corresponds to a solution of~\eqref{SDPp} with additional constraints,
 and hence
robustness of pre-whitened SPA follows from the result above (Section~\ref{prewsdp}).
However, this result is not tight and we provide a tight robustness analysis of pre-whitening (Section~\ref{tight}).
We also provide a robustness analysis of pre-whitening under a standard generative model (Section~\ref{genmod}).

\item In Section~\ref{spaspasec},
we analyzed a preconditioning based on SPA
 itself.
 The idea was proposed in \cite{GM14}, where the resulting method was found to be
extremely fast, and, as opposed to pre-whitening,
 perform
perfectly in the noiseless case and is not affected by the abundances of the different endmembers. Moreover, we are able to improve the theoretical bound on the noise level allowed by SPA by a factor $\kappa(W)$ using this preconditioning.
  \end{enumerate}

Finally, in Section~\ref{ne}, we illustrate these results on synthetic data sets.
In particular, we show that pre-whitening and the SPA-based preconditioning performs competitively with the SDP-based preconditioning while  showing much better runtime performance in practice.

\section{Analysis of Approximate SDP Preconditioning} \label{approxsdpprec}

In this section, we analyze the effect of using an approximate solution $A$ of \eqref{SDPp} instead of the optimal one $A^*$ for preconditioning matrix $\tilde{X}$.
We will say that $A$ is an $\alpha$-approximate solution of \eqref{SDPp} for some $0 < \alpha \leq 1$
if $A$ is a feasible solution of \eqref{SDPp}, that is, $A \in \mathbb{S}^r_+$ and $\tilde{x}_j^T A \tilde{x}_j \leq 1$ for all $j$, and
\[
\det(A) \geq \alpha \det(A^*),
\]
where $A^*$ is the optimal solution of \eqref{SDPp}.
Letting $A = QQ^T$, analyzing the effect of $Q$ as a preconditioning reduces to show that $QW$ is well-conditioned,
that is, to upper bound $\kappa(Q W)$,
which is equivalent to bounding $\kappa(W^T A W)$ since $\kappa(W^T A W) = \kappa(Q W)^2$.
In fact, if $\kappa(QW)$ can be bounded, robustness of preconditioned SPA follows from Corollary~\ref{cor1}.

In \cite{GV13}, a change of variable is performed on the SDP \eqref{SDPp} using $A = W^{-T} C W^{-1}$ to obtain the following equivalent problem
\begin{align}
C^* \quad =  \quad
\argmax_{C \in \mathbb{S}^r_+} \; \; & \det(C) \; \det(W)^{-2} 
\quad
\text{such that }
\quad
 \tilde{x}_j^T
\left( W^{-T} C W^{-1} \right)
 \tilde{x}_j
\leq 1 \; \forall \, j.    \label{SDPc}
\end{align}
Since our goal is to bound $\kappa(W^T A W)$ and $W^T A W = C$, it is equivalent to bound $\kappa(C)$.
It was shown in \cite{GV13} that, for sufficiently small noise level $\epsilon$, $\kappa(W^T A^* W) = \kappa(C^*) = \mathcal{O}(1)$ which implies robustness of SDP-preconditioned SPA; see Theorem~\ref{th2}.
In this section, we analyze
 \eqref{SDPc} directly
and will use the following assumption:
\begin{definition} \label{def1}
The matrix $C$ is an $\alpha$-approximate solution of \eqref{SDPc} for some $0 < \alpha \leq 1$, that is, $C$  is a feasible solution of \eqref{SDPc} and
\[
\det(C) \geq \alpha \det(C^*) .
\]
\end{definition}
Note that $C$ is an $\alpha$-approximate solution of \eqref{SDPc} if and only if
 $A = W^{-T} C W^{-1}$ is an $\alpha$-approximate solution of \eqref{SDPp}. In fact,
$\det(W^{-T} C W^{-1}) = \det(W^{-T}) \det(C) \det(W^{-1})$.

The main result of this section is therefore to show that any $\alpha$-approximate solution of \eqref{SDPc} is well-conditioned. More precisely, we show in Theorem~\ref{lamkap} that,
 for a sufficiently small noise level $\epsilon$,
\[
\kappa(C) \leq \frac{12}{\alpha} .
\]
This will imply, by Corollary~\ref{cor1}, that preconditioned SPA using an approximate solution of~\eqref{SDPp} is robust to noise, given that $\alpha$ is sufficiently close to one; see Theorem~\ref{mainth}.

\subsection{Bounding $\kappa(C)$}

It is interesting to notice that the case $r = 1$ is trivial since $C^*$ is a scalar
 thus has
$\kappa(C) = 1$
 (In fact, all columns of $\tilde{X}$ are multiple of the unique column of  $W \in \mathbb{R}^{m \times 1}$).
Otherwise, since $\kappa(C) \geq 1$, we only need to provide an upper bound for $\kappa(C)$.
The steps of the proof are the following:
\begin{itemize}

\item Derive a lower bound for $\det(C)$ (Lemma~\ref{lemlwb}).

\item Provide an upper bound of $\tr(C)$ (Lemma~\ref{lamtr}).

\item Combine the two bounds above to bound $\kappa(C)$.
In fact, we prove in Lemma~\ref{lamopt} that the condition number $\kappa(C)$ of
an $r$-by-$r$ matrix $C$ with $\tr(C) \leq \beta$ and $\det(C) \geq \gamma$ can be bounded above; see Equation~\eqref{ubkapC}. 
\end{itemize} 
The lower bound for $\det(C)$ and the upper bound for $\tr(C)$ follow  directly from results in \cite{GV13}.
\begin{lemma} 
\label{lemlwb}
If $\tilde{X} = W+N$ where $X$ satisfies Assumption~\ref{ass1} with $m = r$, then any $\alpha$-approximate solution $C$ of \eqref{SDPc} satisfies
\begin{equation} \label{detClow}
\det(C) \geq \alpha \left( 1 + \frac{\epsilon}{\sigma_{\min}(W)} \right)^{-2r} .
\end{equation}
\end{lemma}
\begin{proof}
In \cite[Lemma 1]{GV13}, it was proved that the optimal solution $C^*$ of \eqref{SDPc} satisfies
\[
\det(C^*) \geq \left( 1 + \frac{\epsilon}{\sigma_{\min}(W)} \right)^{-2r} .
\]
 Hence,
the result follows directly from Definition~\ref{def1}.
\end{proof}

\begin{lemma} 
\label{lamtr}
If $\tilde{X} = W+N$ where $X$ satisfies Assumption~\ref{ass1} with $m = r$, and
\[
\epsilon \leq \frac{\sigma_{\min}(W)}{8 r \sqrt{r}},
\]
then any feasible solution $C$ of \eqref{SDPc} satisfies $\tr(C) \leq r+1$.
\end{lemma}
\begin{proof}
See Lemma 2 and proof of Lemma 3 in \cite{GV12}.
\end{proof}

\begin{lemma} \label{lamopt}
The optimal value
\begin{align*}
\kappa^* \quad = \quad
\max_{\lambda \in \mathbb{R}^r} \quad \quad \frac{\lambda_1}{\lambda_r} \quad \text{ such that }
	&  \quad \sum_{i=1}^r \lambda_i \leq \beta,  \\
	&  \quad \prod_{i=1}^r \lambda_i \geq \gamma, \text{ and } \\
	& \quad \lambda_1 \geq \lambda_2 \geq \dots \geq \lambda_r \geq 0.
\end{align*}
where $\beta \geq r$ and  $0 < \gamma \leq 1$ is given by
\begin{equation} \label{ubkapC}
\kappa^* = \frac{1 + \sqrt{1- \gamma \left(\frac{r}{\beta}\right)^r}}{1 - \sqrt{1- \gamma \left(\frac{r}{\beta}\right)^r}} \; .
\end{equation}
\end{lemma}
\begin{proof}
The proof is given in Appendix~\ref{app1}.
\end{proof}

\begin{theorem} \label{lamkap}
If $\tilde{X} = X + N$ where $X$ satisfies Assumption~\ref{ass1} with $m = r \geq 2$ and $\max_j ||N(:,j)||_2 \leq \epsilon \leq \frac{\sigma_{\min}(W)}{8 r \sqrt{r}}$, then
\[
\kappa(C) \leq \frac{12}{\alpha} ,
\]
where $C$ is an $\alpha$-approximate solution of \eqref{SDPc}.
\end{theorem}
\begin{proof}
By Lemma~\ref{lemlwb}, we have
\[
\det(C)
\geq \alpha \left( 1 + \frac{\epsilon}{\sigma_{\min}(W)} \right)^{-2r}
\geq  \alpha \left( 1 + \frac{1}{8 r \sqrt{r}} \right)^{-2r}
\geq  \alpha \underbrace{\left( 1 + \frac{1}{64} \right)^{-4}}_{= \eta} = \alpha \eta,
\]
 where the third inequality above is obtained by the fact that
$\left( 1 + \frac{1}{8 r \sqrt{r}} \right)^{-2r}$ is increasing in $r$ for $r \geq 1$.
 Also,
by Lemma~\ref{lamtr},
 we have
$\tr(C) \leq r+1$.
Combining these results with Lemma~\ref{lamopt}
 (via setting $\gamma = \eta \alpha$ and $\beta = r+1$) yields
\begin{align*}
\kappa(C)
& \leq \frac{1 + \sqrt{1- \alpha \eta  \left(\frac{r}{r+1}\right)^r}}{1 - \sqrt{1- \alpha \eta  \left(\frac{r}{r+1}\right)^r}}
 \leq \frac{1 + \sqrt{1-  \frac{\alpha}{3}}}{1 - \sqrt{1-\frac{\alpha}{3} }},
\end{align*}
where the second inequality above is obtained by the facts that
$\left(\frac{r}{r+1}\right)^r$ is nonincreasing in $r$ and its limit is given by
\[
\lim_{r \rightarrow \infty} \left(\frac{r}{r+1}\right)^r
= \lim_{r \rightarrow \infty} \left(1 - \frac{1}{r+1}\right)^r = \frac{1}{e} = 0.3679,
\]
 and that
$\frac{\eta}{e} =    \frac{\left( 1 + \frac{1}{64} \right)^{-4}}{e} \geq \frac{1}{3}$.
Finally, since the function $\frac{1+y}{1-y}$ is increasing for $0 \leq y < 1$ and $\sqrt{1-x} \leq 1 - \frac{x}{2}$ for all $0 \leq x \leq 1$, we have
\[
\frac{1 + \sqrt{1-  \frac{\alpha}{3}}}{1 - \sqrt{1-\frac{\alpha}{3} }}
\leq
\frac{1 + 1-  \frac{\alpha}{6}}{1 - 1+\frac{\alpha}{6} }
\leq \frac{12}{\alpha} - 1.
\]
\end{proof}

\subsection{Robustness of SPA Preconditioned with an Approximate SDP Solution}

The upper bound on $\kappa(C)$ (Theorem~\ref{lamkap}) proves that the preconditioning generates a well-conditioned near-separable matrix for $\alpha$ sufficiently close to one.
Hence any good approximation of \eqref{SDPp} allows us to obtain more robust near-separable NMF algorithms.
In particular, we have the following result for SPA.
\begin{theorem} \label{mainth}
Let $\tilde{X} = X+N$ where $X$ satisfies Assumption~\ref{ass1} with $m = r$, $W$
 has
full rank and the noise $N$ satisfies $\max_j ||N(:,j)||_2 \leq \epsilon$. Let also $Q \in \mathbb{R}^{r \times r}$ be such that $A = Q^TQ$ where $A$ is an $\alpha$-approximate solution of \eqref{SDPp}.
If
\[
\epsilon \leq \mathcal{O} \left( \min\left( \frac{1}{r}, \alpha^{3/2}  \right) \frac{\sigma_{\min}(W)}{\sqrt{r}} \right),
\] 
then SPA applied on matrix $Q\tilde{X}$ identifies indices corresponding to the columns of $W$ up to error  $\mathcal{O} \Big( \epsilon \kappa(W) \alpha^{-3/2} \Big)$.
\end{theorem}
\begin{proof}
This follows from Corollary~\ref{cor1} and Theorem~\ref{lamkap}.
Let $Q$ be such that $A = Q^TQ$ where $A$ is an $\alpha$-approximate solution of \eqref{SDPp}.
 Since $C = W^T A W$ is an  $\alpha$-approximate solution of \eqref{SDPc}, we have that
\[
(QW)^T QW = W^T Q^T Q W = W^T A W = C.
\]
 Using
$\sigma_i(QW) = \sqrt{\lambda_i(C)}$ for all $i$,
 and by Theorem~\ref{lamkap} for which we need to assume that \mbox{$\epsilon \leq \mathcal{O} \left( \frac{\sigma_{\min}(W)}{r \sqrt{r}} \right)$},
 we further obtain
\[
\kappa(QW) \leq \sqrt{ 12 \alpha^{-1}} = 2 \sqrt{3} \alpha^{-1/2} .
\]
 Applying the above bound to Corollary~\ref{cor1} leads to the desired result: $\epsilon \leq \mathcal{O} \left( \frac{\sigma_{\min}(W)}{\kappa^3 (QW) \sqrt{r}} \right)$ leads to an error proportional to  $\mathcal{O} \left( \epsilon \kappa(W) \kappa^3(QW) \right)$.
\end{proof}

Theorem~\ref{mainth} shows that the SDP-based preconditioning does not require a high accuracy solution of the SDP.
For example, compared to the optimal solution, any $\frac{1}{2}$-approximate solution would not change the upper bound on the noise level for $r \geq 3$ (since $0.5^{3/2} \geq 1/3$) while it would increase the error by a factor smaller than three. This paves the way to the design of much faster preconditionings, solving \eqref{SDPp} only approximately: this is a topic for further research.
Moreover, this analysis will allow us to understand better two other preconditionings: pre-whitening (Section~\ref{prewsec}) and
 SPA-based
preconditioning (Section~\ref{spaspasec}).

\section{Pre-Whitening} \label{prewsec}

Noise filtering and pre-whitening are standard techniques in blind source separation; see, e.g., \cite{CJ10}. They were used in \cite{GV13} as a preconditioning for SPA.
In this section, in light of the results from the previous section, we analyze pre-whitening as a preconditioning. We bound the condition number of $\kappa(QW)$ where $Q$ is the preconditioner obtained from pre-whitening.
In the worst case, $\kappa(QW)$ can be large as it depends on the number of pixels $n$.
However, under some standard generative model, pre-whitening can be
shown to be much more robust.

\subsection{Description}

Let $(U_r,\Sigma_r,V_r) \in \mathbb{R}^{m \times r} \times \mathbb{R}^{r \times r} \times \mathbb{R}^{n \times r}$ be the rank-$r$ truncated SVD of $\tilde{X} = X+N$, so that $\tilde{X}_r = U_r\Sigma_rV_r^T$ is the best rank-$r$ approximation of $\tilde{X}$ with respect to the Frobenius norm. Assuming
 that
the noiseless data matrix $X$ lives in a $r$-dimensional linear space and
 assuming Gaussian noise,
replacing $\tilde{X}$ with $\tilde{X}_r$ allows noise filtering.
Given $\tilde{X}_r$, pre-whitening amounts to keeping only the matrix $V_r$. Equivalently,
it amounts to premultiplying $\tilde{X}_r$ (or $\tilde{X}$)  with $Q = \Sigma_r^{-1}U_r^T$ since
\[
Q \tilde{X}
= \left( \Sigma_r^{-1}U_r^T \right) \left(U \Sigma V^T \right)
= \left( \Sigma_r^{-1}U_r^T \right) \left(U_r \Sigma_r V_r^T \right)
= Q \tilde{X}_r
= V_r^T ,
\]
where $(U,\Sigma,V)$ is the full SVD of $\tilde{X}$ (recall $U(:,1$:$r) = U_r$, and
 that
the columns of $U$ and $V$ are orthonormal); see Alg.~\ref{prew}.
\renewcommand{\thealgorithm}{NF-PW}
\algsetup{indent=2em}
\begin{algorithm}[ht!]
\caption{-- Noise Filtering and Pre-Whitening\label{prew}}
\begin{algorithmic}[1]
\REQUIRE Matrix $\tilde{X} = X + N$ with $X$ satisfying Assumption~\ref{ass1}, rank $r$.
\ENSURE Preconditioner $Q$.
    \medskip

\STATE $[U_r,\Sigma_r,V_r]$ = rank-$r$ truncated SVD($\tilde{X}$). \vspace{0.1cm}

\STATE $Q = \Sigma_r^{-1} U_r^T$.

\end{algorithmic}
\end{algorithm}
In \cite{GV13}, Alg.~\ref{prew} is used as an heuristic preconditioning, and performs similarly as the SDP-based preconditioning. Alg.~\ref{prew} requires $\mathcal{O}(mnr)$ operations to compute the rank-$r$ truncated SVD of $\tilde{X}$.

\subsection{Link with the SDP-based Preconditioning} \label{prewsdp}

For simplicity, let us consider the case $m=r$ (or, equivalently, assume that noise filtering has already been applied to the data matrix). In that case, the preconditioner $Q$ given by pre-whitening is
\[
Q = \Sigma_r^{-1} U_r^T  = \left(\tilde{X} \tilde{X}^T\right)^{-1/2} ,
\]
where $(.)^{-1/2}$ denotes the inverse of the square root of a positive definite matrix, that is, $Q = S^{-1/2} \iff S = \left(Q^{T} Q\right)^{-1}$ (which is unique up to orthogonal transformations).
In fact, for $m = r$, $\tilde{X} = U_r \Sigma_r V_r^T$ hence $\tilde{X} \tilde{X}^T = (U_r \Sigma_r)(U_r \Sigma_r)^T$.
We have the following well-known result:
\begin{lemma} \label{pwopti}
Let $\tilde{X} \in \mathbb{R}^{r \times n}$ be of rank $r$. Then $B^* = \left(\tilde{X}\tilde{X}^T\right)^{-1}$ is the optimal solution of
\begin{align}
\max_{B \in \mathbb{S}^r_+} \; \; & \; \det(B)
\quad  \text{ such that } \quad
								  \sum_{i=1}^n
 \tilde{x}^T_j
B
 \tilde{x}_j
 \leq r  .    \label{SDPb}
\end{align}
Moreover, $B^*$ satisfies
$ \tilde{x}^T_j B^*  \tilde{x}_j \leq 1 \; \forall \, j$.
\end{lemma}
\begin{proof}
Observing that the objective can be replaced with $\log \det (B)$, and that the constraint will be active at optimality (otherwise $B$ can be multiplied by a scalar larger than one), any optimal solution has to satisfy the following first-order optimality conditions:
\[
B^{-1} = \lambda \sum_{i=1}^n \tilde{x_j} \tilde{x_j}^T = \lambda \tilde{X}\tilde{X}^T,
\quad
\text{ and }
\quad
\sum_{i=1}^n \tilde{x_j}^T B \tilde{x_j} = r,
\]
where $\lambda \geq 0$ is the Lagrangian multiplier. This implies that $B^* = \lambda^{-1} \left(\tilde{X}\tilde{X}^T\right)^{-1}$. We have 
\begin{align*}
 \sum_{i=1}^n \tilde{x_j}^T B^* \tilde{x_j}
& = \lambda^{-1} \left\langle  \left(\tilde{X}\tilde{X}^T\right)^{-1}, \tilde{X}\tilde{X}^T \right\rangle
 = \lambda^{-1} {\rm tr}\left(  \left(\tilde{X}\tilde{X}^T\right)^{-1}  \tilde{X}\tilde{X}^T \right)
 = \lambda^{-1} {\rm tr}( I_r ) = r \lambda^{-1},
\end{align*} 
 The above equation, together with the condition $\sum_{i=1}^n \tilde{x_j}^T B \tilde{x_j} = r$, imply
$\lambda = 1$, which proves $B^* = (\tilde{X}\tilde{X}^T)^{-1}$.

Denoting $\tilde{X} = U \Sigma V^T$ the compact SVD of $\tilde{X}$ ($V \in \mathbb{R}^{n \times r}$), we obtain  $\tilde{x_j}^T B \tilde{x_j} = ||V(j,:)||_2^2 \leq 1$ since $V$ has orthogonal columns and $n \geq r$.
\end{proof}

A robustness analysis of pre-whitening follows directly from Theorem~\ref{mainth}. In fact, by Lemma~\ref{pwopti}, the
matrix $(\tilde{X} \tilde{X}^T)^{-1}$ is a feasible solution of the minimum volume ellipsoid problem~\eqref{SDPp}.
Moreover, it is optimal up to a factor $\frac{r}{n}$: by Lemma~\ref{pwopti}, the optimal solution of
\begin{align*}
\max_{D \in \mathbb{S}^r_+} \; \; & \; \det(D)
\quad  \text{ such that } \quad
								  \sum_{i=1}^n \tilde{x_j}^T D \tilde{x_j} \leq n  .
\end{align*}
is given by $D^* = \frac{n}{r} B^*$ while the optimal solution $A^*$ of~\eqref{SDPp} is a feasible solution of this problem
(since $\tilde{x_j}^T A^* \tilde{x_j} \leq 1$ $\forall j$)  so that
\[
\det(B^*) \leq \det(A^*)  \leq \det(D^*) = \det\left( \frac{n}{r} B^* \right) = \left( \frac{n}{r} \right)^r \det(B^*)
\]
 and hence
\[
\det(B^*) \geq  \left( \frac{r}{n} \right)^r  \det(A^*).
\]
In other words, $B^*$ is a $\left( \frac{r}{n} \right)^r$-approximate solution of \eqref{SDPp}.
 Combining this result
with Theorem~\ref{mainth}, we obtain
\begin{corollary}
Let $\tilde{X} = X+N$ where $X$ satisfies Assumption~\ref{ass1} with $m = r$, $W$ is
 of
full rank and the noise $N$ satisfies $\max_j ||N(:,j)||_2 \leq \epsilon$. If
\[
\epsilon \leq \mathcal{O} \left( \left( \frac{r}{n} \right)^{\frac{3r}{2}}  \frac{\sigma_{\min}(W)}{\sqrt{r}} \right),
\]
 pre-whitened SPA identifies indices corresponding to the columns of $W$ up to error
$\mathcal{O} \Big( \epsilon \kappa(W) \left( \frac{n}{r} \right)^{\frac{3r}{2}}  \Big)$.
\end{corollary}
This bound is rather bad as $n$ is
 often large
compared to $r$;
 for the hyperspectral unmixing application, we typically have
$n \geq 10^6$ and \mbox{$r \lesssim 30$}.
In the next  subsections,
we provide a tight robustness analysis of pre-whitening,
and analyze pre-whitening under a standard generative model.

\subsection{Tight Robustness Analysis}  \label{tight}

In this
 subsection,
we provide a better robustness analysis of pre-whitening.
 More
precisely, we provide a tight upper bound for $\kappa(QW)$.
As before, we only consider the case $m=r$.
Under Assumption~\ref{ass1}, we have
\[
\tilde{X} = X + N = W [I_r, H'] + N = W \left( [I_r, H'] + W^{-1} N \right) = W \left( [I_r, H'] + N' \right) = W Y,
\]
where we denote $N' = W^{-1} N$ and $Y = [I_r, H'] + N'$.
Recall that the conditioner $Q$ given by pre-whitening is $Q = (\tilde{X}\tilde{X}^T)^{-1/2}$.
Hence, the condition number of $QW$ will be equal to the square root of the condition number of $Y$. In fact,
\[
Q = \left(\tilde{X}\tilde{X}^T\right)^{-1/2} = \left( W Y Y^T W^T\right)^{-1/2} =
 \left( Y Y^T \right)^{-1/2} W^{-1},
\]
so that $\kappa(QW) = \kappa\left( (YY^T)^{-1/2} \right) = \kappa\left( Y \right)$.
Therefore, to provide a robustness analysis of pre-whintening, it is sufficient to bound $\kappa\left( Y \right)$.
In the next lemma, we show that
$\kappa\left( Y \right) \leq \mathcal{O}(\sqrt{n-r+1})$,
 which implies $\kappa(QW) \leq \mathcal{O}(\sqrt{n-r+1})$  and will lead to an error bound that is much smaller than that derived in the previous subsection.
\begin{lemma} \label{pwcondi}
Let $H' \in \mathbb{R}^{r \times (n-r)}$ be a nonnegative matrix with $||H'(:,j)||_1 \leq 1$ for all $j$,
and let $N' \in \mathbb{R}^{r \times n}$ satisfy $||N'||_2 \leq \delta < 1$.
Then,
\[
\kappa \left( [I_r, H'] + N' \right) \leq \frac{\sqrt{1+n-r} + \delta}{1-\delta}  \, .
\]
\end{lemma}
\begin{proof}
Let us first show that
\[
\sigma_{\min}\left( [I_r, H'] \right) \geq 1,
\quad \text{ and } \quad
\sigma_{\max}\left( [I_r, H'] \right) \leq \sqrt{1+n-r}.
\]
We have
\[
\sigma_{p}\left( [I_r, H'] \right)^2 = \lambda_{p} \left( [I_r, H'] [I_r, H']^T \right)
\]
where $p = \min$ or $\max$,
 and note that
\[
[I_r, H'] [I_r, H']^T = I_r + H'H'^T.
\]
 Using $H'H'^T \succeq 0$, one easily gets
\[
\lambda_{\min} \left( I_r + H'H'^T \right) \geq \lambda_{\min} \left( I_r \right) = 1.
\] 
 Also, we have
\[
\lambda_{\max} \left( I_r + H'H'^T \right)  
\leq  \lambda_{\max}(I_r) + \lambda_{\max}(H'H'^T) \leq 1 + n - r,
\] 
as 
 $\lambda_{\max}(H'H'^T) \leq \tr(H' H'^T)$ and 
\[
\tr(H' H'^T) = \tr \left( \sum_j H'(:,j) H'(:,j)^T \right)
 = \sum_j \tr \left(  H'(:,j) H'(:,j)^T \right)
= \sum_j  ||H'(:,j)||_2^2 \leq n-r,
\]
since $||H'(:,j)||_2 \leq ||H'(:,j)||_1 \leq 1$. 
Finally, using the singular value perturbation theorem (Weyl; see, e.g., \cite{GV96}), we have 
\[
\sigma_p([I_r, H') - ||N||_2 \leq \sigma_p([I_r, H'] + N) \leq \sigma_p([I_r, H') + ||N||_2,
\]
for $p = \min$ or $\max$, and since $||N||_2 \leq \delta < 1$, we obtain
\[
\kappa \left( [I_r, H'] + N \right) \leq \frac{\sqrt{1+n-r} + \delta}{1-\delta}.
\]
\end{proof}

This bound allows us to provide a robustness analysis of pre-whitening.
\begin{theorem} \label{pwtight}
Let $\tilde{X} = X + N$ where $X$ satisfies Assumption~\ref{ass1} with $m=r$, $W$  has
full rank and  the noise $N$ satisfies
\mbox{$\max_j ||N(:,j)||_2 \leq \epsilon$}.
If $\epsilon < \mathcal{O} \left(\frac{\sigma_{\min}(W)}{(n-r+1)^{3/2} \sqrt{r}} \right)$,
pre-whitened SPA identifies the columns of $W$ up to error $\mathcal{O}\left((n-r+1)^{3/2} \epsilon \kappa(W) \right)$.
\end{theorem}
\begin{proof}
This follows from Corollary~\ref{cor1} and Lemma~\ref{pwcondi}. We have
\[
||N'||_2
= ||W^{-1} N||_2
\leq \frac{||N||_2}{\sigma_{\min}(W)}
\leq \frac{\sqrt{n} \max_j ||N(:,j)||_2}{\sigma_{\min}(W)}
\leq \epsilon  \frac{\sqrt{n}}{\sigma_{\min}(W)} \leq \mathcal{O}\left(\frac{1}{n}\right),
\] 
 and as a result Lemma~\ref{pwcondi} can be applied to obtain $\kappa(QW) \leq \mathcal{O}(\sqrt{n-r+1})$ for pre-whitening.
By plugging the above bound into Corollary~\ref{cor1}, Theorem~\ref{pwtight} is obtained. 
\end{proof}

The bounds of Theorem~\ref{pwtight} are tight. In fact, if, except for the pure pixels, all pixels contain the same endmember, say the $k$th, then all columns of $H'$ are equal to the $k$th column of the identity matrix, that is, $H'(:,j) = I_r(:,k) := e_k$ for all $j$. 
Therefore,
\[
\kappa \left( Y \right) = \kappa \left( [I_r, H' ] \right) = \sqrt{1+n-r}
\]
since $YY^T = (I_r + H'H'^T) = I_r + (n-r) e_k e_k^T$ is a diagonal matrix with $(YY^T)_{ii} = 1$ for all $i \neq k$ and $(YY^T)_{kk} = 1+n-r$.

This indicates that pre-whitening should perform the worse when one endmember contains most pixels, as this matches the upper bound of Theorem~\ref{pwtight}. However, if the pixels are relatively well-spread in the convex hull of the endmembers, then pre-whitening may perform well.
 This will be proven in the next subsection, wherein the robustness of pre-whitening under a standard generative model is analyzed.

\subsection{Robustness under a Standard Generative Model}  \label{genmod}

We continue our analysis by considering a standard generative model in hyperspectral unmixing.
 We again consider $m=r$, and the generative model is described as follows.

\begin{assumption} \label{diricass}
The near-separable matrix $\tilde{X} = WH + N$ is such that
\begin{enumerate}
\item[(i)] $W$
 is of full rank.

\item[(ii)] $H(:,j)$ is i.i.d.\@ following a Dirichlet distribution with parameter $\alpha = (\alpha_1, \dots, \alpha_r) > 0$, for all $j$.
 Also, without loss of generality, it will be assumed that $\alpha_1 \geq \alpha_2 \geq \dots \geq \alpha_r$.

\item[(iii)] $N(:,j)$ is i.i.d.\@ with mean zero and covariance $\mathbb{E}\left[ N(:,j)N(:,j)^T \right] = \sigma^2_N I$, for all $j$ (Gaussian noise).

\item[(iv)] The number of samples goes to infinity, that is, $n \rightarrow \infty$.
\end{enumerate}
\end{assumption}
Note that the assumption (ii), which models the abundances as being Dirichlet distributed, is a popular assumption in the HU context; see the literature, e.g., \cite{NB12}.
In particular, the parameter $\alpha$ characterizes how the pixels are spread.
To describe this,
let us consider a simplified case where $\beta := \alpha_1 = \ldots = \alpha_r$; i.e., symmetric Dirichlet distribution.
We have the following phenomena:
if $\beta= 1$, then $H(:,j)$'s are uniformly distributed over the unit simplex;
if $\beta < 1$ and $\beta$ decreases, then $H(:,j)$'s are more concentrated around the vertices (or pure pixels) of the simplex;
if $\beta > 1$ and $\beta$ increases, then $H(:,j)$'s are more concentrated around the center of the simplex.
In fact, $\beta \rightarrow 0$ means that $H(:,j)$'s contain only pure pixels in the same proportions.
It should also be noted that we do not assume the separability or pure-pixel assumption, although the latter is implicitly implied by Assumption 2. 
Specifically, under the assumptions (ii) and (iv), for every endmember there exists pixels that are arbitrarily close to the pure pixel in a probability one sense. 

Now, our task is to prove a bound on $\kappa(QW)$ under the above statistical assumptions,
thereby obtaining implications on how pre-whitened SPA may perform with respect to the abundances' distribution (rather than in the worst-case scenario).
To proceed,
 we formulate the pre-whitening preconditioner as
\[
Q = R^{-1/2}
\quad  \text{ where } \quad
R = \frac{1}{n} \tilde{X} \tilde{X}^T = \frac{1}{n} \sum_{j=1}^n \tilde{X}(:,j) \tilde{X}(:,j)^T.
\]
For $n \rightarrow \infty$, we have
\[
R = \mathbb{E}\left[ \tilde{X}(:,j) \tilde{X}(:,j)^T \right].
\]
Also, under Assumption 2, the above correlation matrix can be shown to be
\[
\mathbb{E}\left[ \tilde{X}(:,j) \tilde{X}(:,j)^T \right] = W \Phi W^T + \sigma^2_N I,
\quad  \text{ where } \quad
\Phi = \mathbb{E}\left[ H(:,j) H(:,j)^T \right].
\]
We have the following lemma.

\begin{lemma} \label{philem}
Under Assumption~\ref{diricass}, the matrix $\Phi = \mathbb{E}\left[ H(:,j) H(:,j)^T \right]$ is given by
\begin{equation} \label{eq:Phi}
\Phi = \frac{1}{\alpha_0 (\alpha_0 + 1)} \left(D + \alpha \alpha^T \right) ,
\end{equation}
where $D = {\rm Diag}(\alpha_1,\ldots,\alpha_r)$ and $\alpha_0 = \sum_{i=1}^r \alpha_i$.
Also, the largest and smallest eigenvalues of $\Phi$ are bounded by
\[
\lambda_{\max}(\Phi)
\leq
u
\coloneqq
 \frac{\alpha_1 + ||\alpha||_2^2}{\alpha_0 (\alpha_0 + 1)},
\quad
\text{ and }
\quad
\lambda_{\min}(\Phi)
\geq
\ell
\coloneqq
\frac{\alpha_r}{\alpha_0 (\alpha_0 + 1)},
\]
respectively.
\end{lemma}
\begin{proof}
It is known that for a random vector $x \in \mathbb{R}^r$ following a Dirichlet distribution of parameter $\alpha$, its means and covariances are respectively given by
\[
\mathbb{E}[x_i] = \frac{\alpha_i}{\alpha_0},
\quad \text{ and } \quad
\text{cov}[x_i,x_j] = \left\{ \begin{array}{cc}
- \frac{\alpha_i \alpha_j}{\alpha_0^2 (\alpha_0 + 1)} & \text{if } i \neq j, \\
 \frac{\alpha_i (\alpha_0 - \alpha_i)}{\alpha_0^2 (\alpha_0 + 1)} & \text{if } i = j.
\end{array} \right.
\]
From the above results, we get
\begin{align*}
\Phi_{ii} = \mathbb{E}[x_i^2]
& = \left( \frac{\alpha_i}{\alpha_0} \right)^2 + \frac{\alpha_i (\alpha_0 - \alpha_i)}{\alpha_0^2 (\alpha_0 + 1)}
=  \frac{\alpha_i^2 + \alpha_i}{\alpha_0 (\alpha_0 + 1)} ,
\end{align*}
and for $i \neq j$,
\[
\Phi_{ij} =  \mathbb{E}[x_i x_j] = \frac{\alpha_i \alpha_j}{\alpha_0^2} -  \frac{\alpha_i \alpha_j}{\alpha_0^2 (\alpha_0 + 1)}
= \frac{ \alpha_i \alpha_j }{\alpha_0 (\alpha_0 + 1)},
\]
which lead to
 \eqref{eq:Phi}.
The bounds on the eigenvalues follows from the fact that for any  $A, B \in \mathbb{S}^r$,
$\lambda_{\max}(A+B) \leq \lambda_{\max}(A)+ \lambda_{\max}(B)$
and
$\lambda_{\min}(A+B) \geq \lambda_{\min}(A)+ \lambda_{\min}(B)$.
\end{proof}

From Lemma~\ref{philem}, we deduce the following result.

\begin{theorem} \label{thken}
Consider preconditioning via pre-whitening.
Under Assumption~\ref{diricass}, the condition number of
$QW$ is bounded by
\[
\kappa(QW) \leq \kappa(W) \sqrt{ \frac{ u \, \sigma_{\min}^2(W) + \sigma_N^2}{\ell \, \sigma_{\max}^2(W)  + \sigma_N^2} }
\]
where
\[
u =  \frac{\alpha_1 + ||\alpha||_2^2}{\alpha_0 (\alpha_0 + 1)}
\quad \text{ and } \quad
\ell  = \frac{\alpha_r}{\alpha_0 (\alpha_0 + 1)}.
\]
\end{theorem}
\begin{proof}
By Lemma~\ref{philem}, we have that $\ell \, WW^T  + \sigma^2_N I
\preceq
W\Phi W^T + \sigma^2_N I
\preceq
u \, WW^T  + \sigma^2_N I$.
It follows that
\[
\left( u \, WW^T  + \sigma^2_N I  \right)^{-1}
\preceq
\left( W\Phi W^T + \sigma^2_N I  \right)^{-1}
\preceq
\left( \ell \, WW^T  + \sigma^2_N I  \right)^{-1}.
\]
Consider $W^T Q^T Q W = W^T \left( W\Phi W^T + \sigma^2_N I  \right)^{-1} W$. Letting $W = U \Sigma V^T$ be the SVD of $W$, we obtain
\begin{align*}
W^T Q^T Q W & \preceq  W^T \left( \ell \, WW^T  + \sigma^2_N I  \right)^{-1} W \\
& =  V \Sigma U^T \left( U (\ell \Sigma^2   + \sigma^2_N I ) U^T  \right)^{-1} U \Sigma V^T \\
& =  V \Sigma  \left(\ell \Sigma^2   + \sigma^2_N I \right)^{-1} \Sigma V^T  \\
& =  V
\left(
\begin{array}{ccc}
\ddots & & \\
& \frac{\sigma_i^2(W) }{\ell \sigma_i^2(W) + \sigma^2_N} & \\
& & \ddots
\end{array} \right)
V^T  .
\end{align*}
Hence $\frac{\sigma_{\max}^2(W) }{\ell \sigma_{\max}^2(W) + \sigma^2_N}$ is an upper bound for the largest eigenvalue of $W^T Q^T Q W$. Using the same trick, we obtain  $\frac{\sigma_{\min}^2(W) }{u \sigma_{\min}^2(W) + \sigma^2_N}$ as a lower bound for the smallest eigenvalue of $W^T Q^T Q W$, which gives the result.
\end{proof}

Combining Corollary~\ref{cor1} with Theorem~\ref{thken} implies robustness of pre-whitening combined with SPA under the aforementioned
generative model.
It is particularly interesting to observe that, assuming $\sigma_{\min}(W) \gg \sigma_N$, we have
\begin{equation}
\kappa(QW)
\lesssim \kappa(W) \sqrt{ \frac{ u \, \sigma_{\min}^2(W) }{\ell \, \sigma_{\max}^2(W) } }
= \sqrt{ \frac{ u  }{\ell } }
 = \sqrt{ \frac{\alpha_1 + ||\alpha||_2^2}{\alpha_r}} .
\label{eq:kappa_QW_prewhit_approx}
\end{equation}
As can be seen, the approximate bound above does not depend on the conditioning of $W$---which is appealing when we plug it into Corollary~\ref{cor1} to obtain its provable SPA error bound.
That said, one should note that $\alpha$, which characterizes how the abundances are spread, plays a role.
To get more insight, consider again the symmetric distribution case $\beta := \alpha_1 = \ldots = \alpha_r$.
Equation~\eqref{eq:kappa_QW_prewhit_approx} reduces to
\[
\kappa(QW)
\lesssim \sqrt{ \frac{\beta + r \beta^2}{\beta}} = \sqrt{1 + r \beta}.
\]
We see that fixing $r$, a smaller (respectively larger) $\beta$ implies an improved (respectively degraded) bound---which is quite natural since $\beta$ controls the concentration of data points around the vertices (or pure pixels).
It is also interesting to look at the asymmetric distribution case.
Specifically, consider an extreme case where we fix $\alpha_1,\ldots,\alpha_{r-1}$ and scale $\alpha_r$ to a very small value.
Physically, this means one endmember is present in very small proportions in the data set---a scenario reminiscent of the worst-case scenario identified by the tight robustness analysis in the last subsection.
Then, from \eqref{eq:kappa_QW_prewhit_approx}, one can see that the bound worsens as $\alpha_r$ decreases. In fact, as $\alpha_r$ goes to zero, the $r$th endmember progressively disappears from the data set and hence cannot be recovered:
For $\alpha_r \rightarrow 0$, \eqref{eq:kappa_QW_prewhit_approx} becomes unbounded, which matches the result  in Theorem~\ref{pwtight} for $n \rightarrow \infty$ (in a worst-case scenario).

To conclude, SPA preconditioned by pre-whitening can yield good performance if the abundances are more uniformly spread and a good population of them is close to the pure pixels.
On the other hand, one will expect deteriorated performance if one of the endmembers exhibits little contributions in the data set, or if most of the pixels are heavily mixed.

\section{SPA-based Heuristic Preconditioning} \label{spaspasec}

 As discussed previously,
the intuition behind designing a good preconditioner
is to find the left inverse $W^{\dagger}$ of $W$ in the ideal case, or to efficiently approximate $W^{\dagger}$ in practice.
This reminds us that the original SPA (or SPA without preconditioning) can extract $W$ exactly in the noiseless case, and approximately in the noisy case (Th.~\ref{th1}).
Hence, a possible heuristic, which has been explored in \cite{GM14}, is as follows:
\begin{enumerate}
\item Identify approximately the columns of $W$ among the columns of $\tilde{X}$ using SPA, that is, identify an index set $\mathcal{K}$ such that $W \approx \tilde{X}(:,\mathcal{K})$
(note that other pure-pixel search algorithms could be used), and
\item Compute $\tilde{X}(:,\mathcal{K})^{\dagger} = \left( \tilde{X}(:,\mathcal{K})^T \tilde{X}(:,\mathcal{K})\right)^{-1} \tilde{X}(:,\mathcal{K})^T$ (this is pre-whitening of $\tilde{X}(:,\mathcal{K})$). 
\end{enumerate}
The computational cost is the one of SPA which requires $2mnr + \mathcal{O}(mr^2)$ operations \cite{GV12}, plus the one of the SVD of $\tilde{X}(:,\mathcal{K})$ which requires $\mathcal{O}(mr^2)$ operations.
 Hence, this SPA-based preconditioning is simple and computationally very efficient.
Our interest with the SPA-based preconditioning is fundamental.
We will draw a connection between the SPA-based preconditioning and the (arguably ideal) SDP-based preconditioning.
Then, a robustness analysis will be given.

Note that, for $m = r$,
the preconditioning is given by $\tilde{X}(:,\mathcal{K})^{-1}$, while $\left( \tilde{X}(:,\mathcal{K})\tilde{X}(:,\mathcal{K})^T\right)^{-1}$ is the optimal solution of
\begin{align}
P^* \quad = \quad
\argmax_{P \in \mathbb{S}^r_+} \; \; & \; \det(P)
\quad  \text{ such that } \quad
									\tilde{x_j}^T P \tilde{x_j} \leq 1 \; \forall \, j \in \mathcal{K} ,      \label{SDPpp}
\end{align}
see \cite[Th.4]{GV13} (this also follows from Lemma~\ref{pwopti}).
Hence our heuristic can be seen as a relaxation of the SDP~\eqref{SDPp} where we have selected a subset of the constraints using SPA. Moreover, by letting $A^*$ be the optimal solution of  \eqref{SDPp}, we have
\[
\det(P^*) \geq \det(A^*).
\] 
Therefore, we can easily provide an a posteriori robustness analysis, observing that
\[
\det(P^*) \geq \det(A^*) \geq \det\left( \frac{P^*}{\max_j \tilde{x_j}^T P^* \tilde{x_j}}\right)
\]
since $\frac{P^*}{\max_j \tilde{x_j}^T P^* \tilde{x_j}}$ is a feasible solution of \eqref{SDPp}.
Therefore, by denoting $\beta = \max_j \tilde{x_j}^T P^* \tilde{x_j}$, $\frac{1}{\beta} P^*$ is a $\frac{1}{\beta^r}$-approximate solution of  SDP~\eqref{SDPp}, and we can apply Theorem~\ref{mainth}.

\begin{remark}
Note that the active set method for~\eqref{SDPp} proposed in \cite{GV13}
implicitly uses the above observation. In fact, the set of initial constraints were selected using SPA, and updated by adding the most violated constraints at each step (that is the constraints corresponding to the largest $\tilde{x_j}^T P^* \tilde{x_j}$).
\end{remark}

We have observed that extracting more than $r$ columns with SPA sometimes gives a better preconditioning. Intuitively, extracting more columns allows to better assess the way the columns of $\tilde{X}$ are spread in space: at the limit, if all columns are extracted, this is exactly Alg.~\ref{prew}. Hence we have added a parameter $r \leq p \leq \min(m,n)$; see Alg.~\ref{pspa}. Note that SPA cannot extract more than $\rank(\tilde{X})$ indices. Therefore, in the noiseless case ($\rank(\tilde{X}) = r$), the SPA-based preconditioning performs perfectly for any $p \geq r$. 
\renewcommand{\thealgorithm}{SPA-Prec}
\algsetup{indent=2em}
\begin{algorithm}[ht!]
\caption{-- SPA-based Preconditioning  \cite{GM14} \label{pspa}}
\begin{algorithmic}[1]
\REQUIRE Matrix $\tilde{X} = X + N$ with $X$ satisfying Assumption~\ref{ass1}, rank $r$, parameter $r \leq p \leq \min(m,n)$.
\ENSURE Preconditioner $Q$.
    \medskip

\STATE $\mathcal{K} = \text{SPA}(\tilde{X}, p)$.
\STATE $Q = \text{\ref{prew}}\left(\tilde{X}(:,\mathcal{K}), r\right)$.
\end{algorithmic}
\end{algorithm}

\ref{pspa} can be used to  make pure-pixel search algorithms more robust, e.g., SPA. It is kind of surprising: one can use SPA to precondition SPA and make it more robust.

\subsection{Robustness Analysis} \label{raspa}

We can provide the following robustness results for \ref{spa} preconditioned with SPA.
\begin{theorem} \label{thspa}
Let $\tilde{X} = X + N$ where $X$ satisfies Assumption~\ref{ass1} with $m=r$, $W$  has full rank and the noise $N$ satisfies
\mbox{$\max_j ||N(:,j)||_2 \leq \epsilon$}.
If $\epsilon \leq \mathcal{O} \left( \,  \frac{  \sigma_{\min}(W)  }{\sqrt{r} \kappa^2(W)} \right)$, then SPA-based preconditioned \ref{spa} identifies the columns of $W$ up to error $\mathcal{O} \left( \epsilon \, \kappa(W) \right)$.
\end{theorem}
\begin{proof} Let us denote $\tilde{W} = \tilde{X}(:, \mathcal{K}) \in \mathbb{R}^{r \times r}$  to be
the matrix extracted by SPA so that the SPA-based preconditioning is given by $Q=\tilde{W}^{-1}$. Hence, using Corollary~\ref{cor1}, it remains to prove that
\[
\kappa \left( Q W \right)  =
\kappa \left( \tilde{W}^{-1} W \right) = \kappa \left( W^{-1} \tilde{W}  \right) = \mathcal{O}(1).
\]
Let us assume without loss of generality that the columns of $\tilde{W}$ are properly permuted (this does not affect the preconditioning) so that, by Theorem~\ref{th1}, we have
\[
E =  W - \tilde{W}
\quad
\text{ where  } \quad
\max_j || E(:,j) ||_2
\leq \mathcal{O} \left( \epsilon \kappa(W)^2 \right)
\leq \mathcal{O} \left( \frac{ \sigma_{\min}(W)  }{\sqrt{r} } \right) .
\]
This implies that $|| E ||_2  \leq \mathcal{O} \left( \sigma_{\min}(W)  \right)$.
 By
denoting $(U,\Sigma_W,V)$
 to be
the SVD of $W$, we have
\[
W^{-1} \tilde{W}
= V \Sigma_W^{-1} U^T ( U \Sigma_W V^T + E)
= I_r + V \Sigma_W^{-1} U^T E.
\]
Denoting $A = V \Sigma_W^{-1} U^T E$, we also have
\[
\kappa(W^{-1} \tilde{W}) \leq \frac{1+\sigma_{\max}(A)}{1-\sigma_{\max}(A)},
\]
where $||V \Sigma_W^{-1} U^T E||_2 \leq ||\Sigma_W^{-1}||_2 ||E||_2 = \sigma_{\min}(W)^{-1} ||E||_2 \leq \mathcal{O}(1)$ which gives the result. The above bound on the condition number follows from the singular value perturbation theorem; see, e.g., \cite[Cor.~8.6.2]{GV96} which states that, for any square matrix $A = B-B'$, we have $\left| \sigma_i(B) - \sigma_i(B') \right| \leq \sigma_{\max}(A)$ for all~$i$.
\end{proof}

It is interesting to notice that
\begin{itemize}

\item SPA-based preconditioned SPA improves the error bound of SPA by a factor $\kappa(W)$.

\item The theoretical result for SPA-based preconditioned \ref{spa} (Theorem~\ref{thspa}) does not allow higher noise levels than SPA.
However, in practice, it will allow much higher noise levels; see the numerical experiments in Section~\ref{ne}.

\item SPA-based preconditioned \ref{spa} has the same robustness as post-processed \ref{spa} \cite{Ar13} (see also \cite{GV13} for a discussion) although SPA-based preconditioned \ref{spa} is computationally slightly cheaper (post-processed SPA requires $r$ orthogonal projections onto $(r-1)$-dimensional subspaces).
 Moreover, post-processed \ref{spa} was shown to perform only slightly better than SPA \cite{GV13} while SPA-based preconditioned SPA will outperform SPA (in particular, this applies to the synthetic data sets described in Section~\ref{middlep}).

\item The procedure can potentially be used recursively, that is, use the solution obtained by SPA preconditioned with SPA to precondition SPA. However, we have not observed significant improvement      in
    doing so, and the error bound that can be derived is asymptotically the same as for a single-pass SPA-based preconditioning. In fact, Theorem~\ref{thspa} can be easily adapted: the only difference in the proof is that
the upper bound for $||E||_2$ would be better, from $\mathcal{O} \left( \sigma_{\min}(W)  \right)$ to  $\mathcal{O} \left( \frac{\sigma_{\min}(W) }{\kappa(W)}  \right)$, which does not influence $\kappa(W^{-1} \tilde{W})$ being in $\mathcal{O}(1)$.

\end{itemize}

\section{Numerical Experiments} \label{ne}

In this section, we compare the following algorithms:
\begin{itemize}
\item \textbf{SPA}. The successive projection algorithm; see Alg.~\ref{spa}.
\item \textbf{SDP-SPA}. Alg.~\ref{sdpprec} + SPA.
\item \textbf{PW-SPA}. Alg.~\ref{prew} + SPA.
\item \textbf{SPA-SPA}. Alg.~\ref{pspa} ($p=r$) + SPA.
\item \textbf{VCA}. Vertex component analysis (VCA) \cite{ND05}, available at \url{http://www.lx.it.pt/~bioucas}.
\item \textbf{XRAY}. Fast conical hull algorithm, `max' variant \cite{KSK12}.
\end{itemize}
The Matlab code is available at \url{https://sites.google.com/site/nicolasgillis/}.  All tests are preformed using Matlab on a laptop with Intel CORE i5-3210M 2.5GHz CPU and with 6GB RAM.

\subsection{Two-by-Three Near-Separable Matrix}

In this 
 subsection,
we illustrate the effectiveness of the preconditionings on a small near-separable matrix: Let
\[
W = \left(
\begin{array}{cc}
k+1 & k \\ k & k+1 \\
\end{array}
\right)
\]
for some parameter $k \geq 0$.
We have that $\sigma_{\min}(W) = 1$ and $\sigma_{\max}(W) = 2k+1$ hence $\kappa(W) = 2k+1$.
Let us also take
\begin{equation} \label{s2by3}
H = \left(
\begin{array}{ccc}
1 & 0 & 0.5 \\ 0 & 1 & 0.5  \\
\end{array}
\right), \; \tilde{X} = WH + N \; \text{ with } \; N = \delta \; [-W(:,1), \; -W(:,2), \; WH(:,3)].
\end{equation}
 Under the above setup, the following phenomena can be shown:
For $\delta \geq \frac{1}{8 k^2}$, we have $||\tilde{X}(:,1)||_2^2 = ||\tilde{X}(:,2)||_2^2 <||\tilde{X}(:,3)||_2^2$.
Subsequently, SPA will extract $\tilde{X}(:,3)$ as an endmember estimate, which is wrong.
To show this, note that
\[
||\tilde{X}(:,1)||_2 =
||\tilde{X}(:,2)||_2
= (1-\delta) \sqrt{2k^2 + 2k + 1}
\quad
\text{ while }
\quad
||\tilde{X}(:,3)||_2 = (1+\delta) \sqrt{2k^2 + 2k + \frac{1}{2}}.
\]
By the above equations,
the condition $||\tilde{X}(:,1)||_2^2 = ||\tilde{X}(:,2)||_2^2 <||\tilde{X}(:,3)||_2^2$ happens when
\[
\delta \geq
\frac{1}{8 k^2} > \frac{\sqrt{2k^2 + 2k + 1} - \sqrt{2k^2 + 2k + \frac{1}{2}}}{\sqrt{2k^2 + 2k + 1} + \sqrt{2k^2 + 2k + \frac{1}{2}}},
\]
where the second inequality is obtained via
\[
\sqrt{2k^2 + 2k + 1} + \sqrt{2k^2 + 2k + \frac{1}{2}} > 2 k
\quad \text{ and } \quad
\sqrt{2k^2 + 2k + 1} - \sqrt{2k^2 + 2k + \frac{1}{2}} < \frac{1}{4k} .
\]
(The second inequality can be obtained by multiplying the left- and right-hand side by
$\sqrt{2k^2 + 2k + 1} + \sqrt{2k^2 + 2k + \frac{1}{2}}$.) 
 Also, it can be shown that, for any $k$, SDP-SPA will extract correctly the columns of $W$, with error proportional to $\mathcal{O}(\delta k)$ for $\delta \leq \mathcal{O}(1)$.
Figure~\ref{exp1} displays the fraction of columns of $W$ properly identified by the different algorithms for different value of $\delta$: on the left for $k= 10$ and, on the right, for $k = 1000$.
\begin{figure}[ht!]
\begin{center}
\begin{tabular}{cc}
\includegraphics[width=8cm]{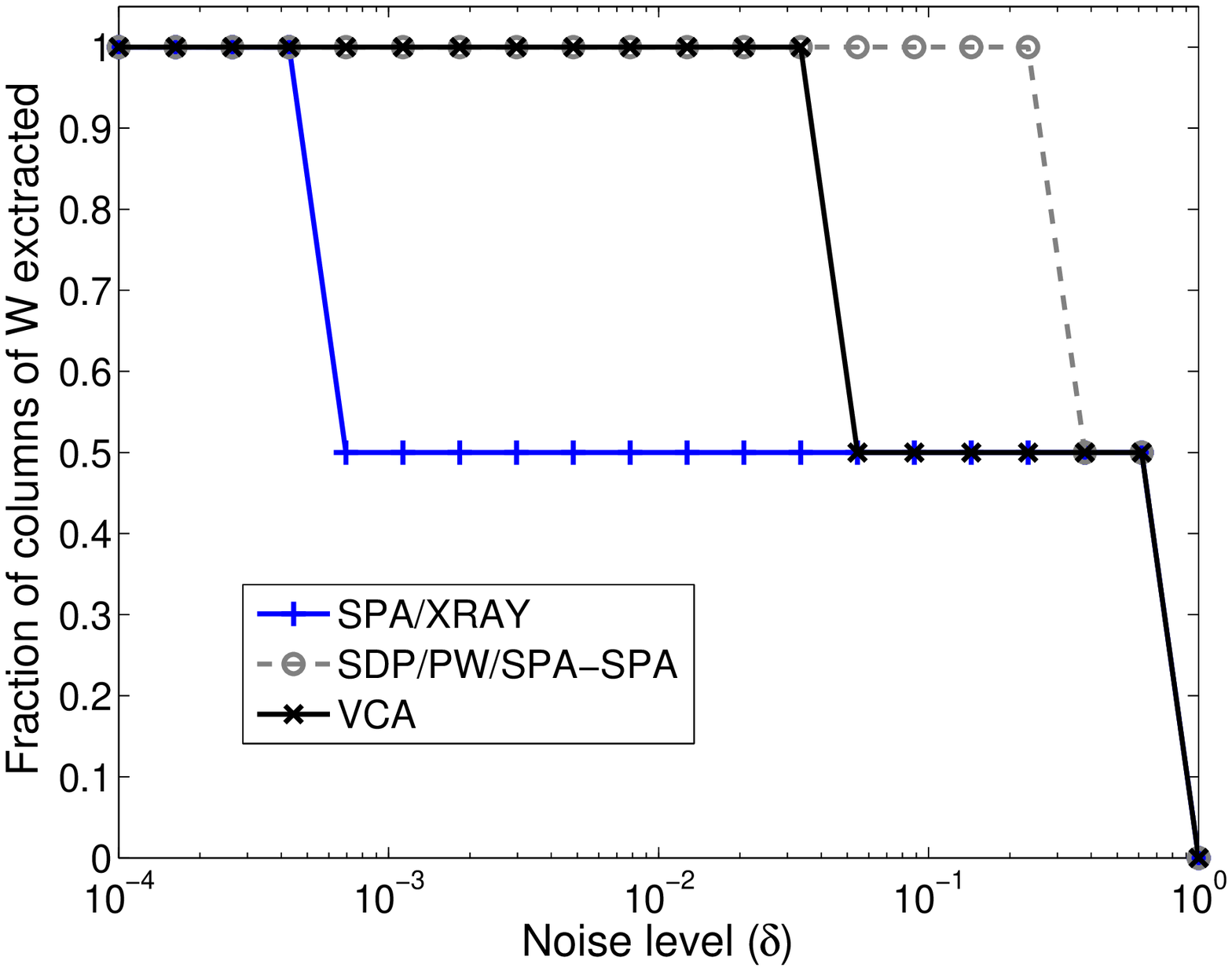} & \includegraphics[width=8cm]{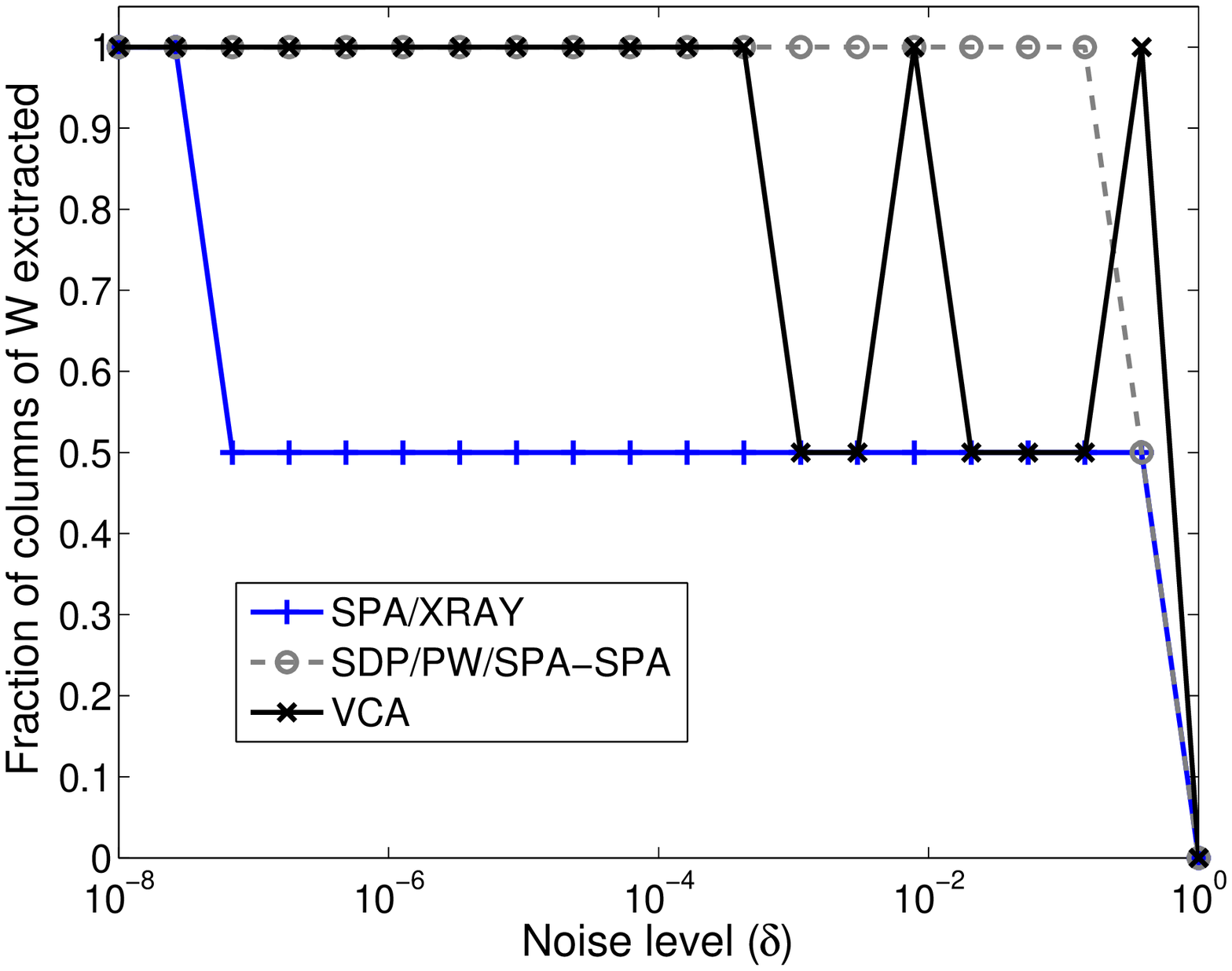} \\
\end{tabular}
\caption{Comparison of the different near-separable NMF algorithms on the matrix from Equation~\eqref{s2by3} with $k = 10$ (left) and $k = 1000$ (right).
}
\label{exp1}
\end{center}
\end{figure}
As explained above, SPA fails to identify properly the two columns of $W$ for any $\delta \leq \mathcal{O}(k^{-2})$ while SDP-SPA works perfectly for $\delta \leq \mathcal{O}(1)$.
It turns out  that all preconditioned variants perform the same, while XRAY performs the same as SPA.
VCA is not deterministic and different runs lead to different outputs. In fact, potentially any column of $\tilde{X}$ can be extracted by VCA for $\delta > 0$.

\subsection{Middle Points Experiment} \label{middlep}

In this
 subsection,
we use the so-called middle points experiment from \cite{GV12}, with $m = 40$, $r=20$ and $n = 210$.
The input matrix satisfies Assumption~\ref{ass1} where each entry of $W$ is generated uniformly at random in [0,1] and $H'$ contains only two non-zero entries equal to 0.5 (hence all data points are in the middle of two columns of $W$). The noise moves the middle points toward the outside of the convex hull of the columns of $W$ with $N(:,j) = \delta \left( X(:,j) - \bar{w} \right)$ where $\bar{w} = \frac{1}{r}\sum_{k=1}^r W(:,k)$ and $\delta$ is the noise parameter; see \cite{GV12} for more details.

For each noise level (from 0 to 0.6 with step 0.01), we generate 25 such matrices and Figure~\ref{xp} and Table~\ref{txp} report the numerical results.
\begin{figure}[ht!]
\begin{center}
\includegraphics[width=\textwidth]{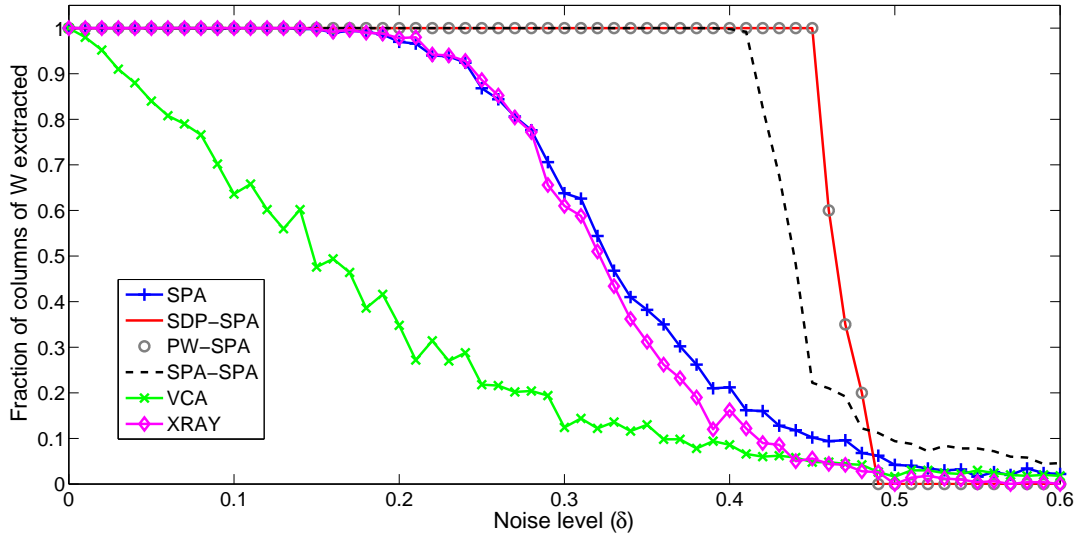}
\caption{Comparison of the different near-separable NMF algorithms on the `Middle Points' experiment.}
\label{xp}
\end{center}
\end{figure}
\begin{table}[ht!]
\begin{center}
\begin{tabular}{|c|c|c|c|}
\hline
  & Robustness  &  Total time (s.) \\ \hline
SPA  &  0.08 & 4  \\
SDP-SPA &  \textbf{0.45}  & 3508 \\
PW-SPA &  \textbf{0.45}   & 34 \\
SPA-SPA & 0.39   &  31 \\
VCA & 0 & 841  \\
XRAY & 0.18  & 743 \\ \hline
\end{tabular}
\caption{Robustness (that is, largest value of $\delta$ for which all columns of $W$ are correctly identified) and total running time in seconds of the different near-separable NMF algorithms.}
\label{txp}
\end{center}
\end{table}
We observe that SPA-SPA is able to improve the performance of SPA significantly: for example, for the noise level $\delta = 0.4$, SPA correctly identifies about 20\% of the columns of $W$ while SPA-SPA does for about 95\%. Note that SPA-SPA is only slightly faster than PW-SPA because (i) $m$ is not much larger than~$r$, and (2) as opposed to SPA-SPA, PW-SPA actually does not need to compute the product $Q\tilde{X}$ where $Q$ is the preconditioning since $Q\tilde{X} = V_r^T$. For large $m$ and $n$, SPA-SPA will be much faster (see the next section for an example). Note also that PW-SPA performs very well (in fact, as well as SDP-SPA) because the data points are well spread in the convex hull of the columns of $W$ hence $\kappa(H)$ is close to one (in fact, it is equal to 1.38 while the average value of $\kappa(W)$ is around 22.5). \\

\subsection{Hubble Telescope}

We use the simulated noisy Hubble telescope hyperspectral image from \cite{PPP06} (with $m = 100$ and $n = 16384$) constituted of 8 endmembers (see Figure~\ref{abmap}). Table~\ref{mrsatim} reports the running time and the mean-removed spectral angle (MRSA) between the true endmembers (of the clean image) and the extracted endmembers. Given two spectral signatures, $x, y \in \mathbb{R}^m$, the MRSA is defined as
\begin{equation} \nonumber
\phi(x,y)
= \frac{100}{\pi}
\arccos \left( \frac{ (x-\bar{x})^T (y-\bar{y}) }{||x-\bar{x}||_2 ||y-\bar{y}||_2} \right)  \in  [0,100].
\end{equation}
Figure~\ref{abmap} displays the abundance maps corresponding to the extracted columns of $W$.
 \begin{table}[ht!]
\begin{center}
\begin{tabular}{|c|cccc|}
\hline
			&     SPA   & SDP-SPA &  PW-SPA & SPA-SPA  \\ \hline
Hon. side		& 	{6.51}  & 6.94 & 6.94 &  \textbf{6.15}\\
Cop. Strip. &   26.83  & 7.46 & \textbf{7.44} &  \textbf{7.44} \\
Green glue & 2.09 & \textbf{2.03} & \textbf{2.03} &  \textbf{2.03} \\
Aluminum & \textbf{1.71} & 1.80 & 1.80  & 1.80 \\
Solar cell  &  \textbf{4.96} & 5.48 & 5.48  & \textbf{4.96} \\
Hon. top  &  2.34  & \textbf{2.30} & \textbf{2.30}  & \textbf{2.30} \\
Black edge   & 27.09  &  {13.16}  & {13.16}  & \textbf{13.13}\\
Bolts  &   \textbf{2.65}  & \textbf{2.65} & \textbf{2.65} &  2.70 \\ \hline
	Average 	&	9.27     &   {5.23} &   {5.23} &   \textbf{5.06}  \\ \hline \hline
 Time (s.)  &  0.05  & 4.74 & 2.18 & 0.37  \\ \hline
\end{tabular}
\caption{MRSA of the identified endmembers with the true endmembers, and running time in seconds of the different preconditioned SPA algorithms. }
\label{mrsatim}
\end{center}
\end{table}

	\begin{figure}[h!]
\begin{center}
\includegraphics[width=\textwidth]{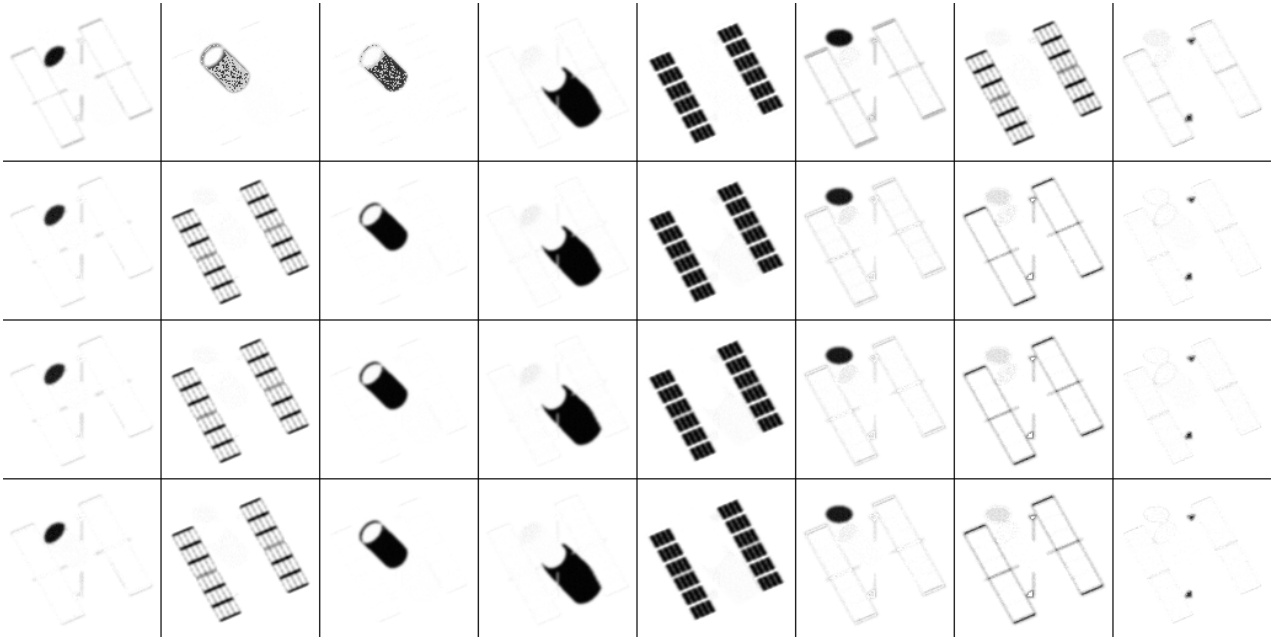}
\caption{The abundance maps corresponding to the endmembers extracted by the different algorithms. From top to bottom: SPA, SDP-SPA, PW-SPA, SPA-SPA. From left to right: Honeycomb side,  Copper Stripping, Green glue, Aluminum,  Solar cell, Honeycomb top, Black rubber edge, and bolts.}
\label{abmap}
\end{center}
\end{figure}

All preconditioned variants are able to identify the 8 materials properly, as opposed to the original SPA.
SPA-SPA performs slightly better than the other pre-conditioned variants while being the fastest.
We do not show the results of VCA and XRAY as they perform very poorly \cite{GV13}.

\section{Conclusion and Further Research}

In this paper, we analyzed several preconditionings for
 making pure-pixel search algorithms more robust to noise:
an approximate SDP, pre-whitening and a simple and fast yet effective SPA-based preconditioning.
 The analyses revealed that these preconditionings, which aim at low-complexity implementation and are suboptimal compared to the ideal SDP preconditioning, actually have provably good error bounds on pure-pixel identification performance.

Further research include the following:
\begin{itemize}

\item Evaluate the preconditionings on real-world hyperspectral images. We have performed preliminary numerical experiments on real-world hyperspectral images and did not observe significant advantages when using the different preconditionings. 
A plausible explanation is that the noise level in such images is usually rather large (in particular larger than the bounds derived in the theorems) and these images contain outliers. Hence, to make preconditionings effective is such conditions, some pre-processing of the data would be necessary; in particular, outlier identification since pure-pixel search algorithms are usually very sensitive to outliers (e.g., VCA, SPA, and XRAY). 

\item Use the preconditioning to enhance other blind hyperspectral unmixing algorithms; for example algorithms which do not require the pure-pixel assumption to hold, e.g., \cite{Li2008,Chan2009,Dias2009}.

\item Analyze theoretically and practically the influence of preconditioning on other pure-pixel search algorithms. For example, the results of this paper directly apply to the successive nonnegative projection algorithm (SNPA) which is more robust and applies to a broader class of matrices ($W$ does not need to be full rank) than SPA \cite{G13}.

\end{itemize}

\bibliographystyle{spmpsci}
\bibliography{Biography}

\normalsize

\appendix

\section{Proof for Lemma~\ref{lamopt}} \label{app1}

We have to prove that $\kappa^* = \frac{1 + \sqrt{1- \gamma \left(\frac{r}{\beta}\right)^r}}{1 - \sqrt{1- \gamma \left(\frac{r}{\beta}\right)^r}}$ satisfies
\begin{align*}
\kappa^* \quad = \quad
\max_{\lambda \in \mathbb{R}^r}
 \quad \frac{\lambda_1}{\lambda_r} \quad
  \text{ such that }
	&   \quad \sum_i \lambda_i \leq \beta,
	  \quad \prod_i \lambda_i \geq \gamma, \quad  \text{ and }
	 \quad \lambda_1 \geq \lambda_2 \geq \dots \geq \lambda_r \geq 0,
\end{align*}
where $\beta \geq r$ and  $0 < \gamma \leq 1$. Note first that the problem is feasible taking $\lambda_i = 1$ for all $i$.

At optimality, the constraint $\sum_i \lambda_i \leq \beta$ must be active (otherwise $\lambda_1$ can be increased to generate a strictly better solution),  the constraint $\prod_i \lambda_i \geq \gamma$ must also be active (otherwise $\lambda_r$ can be decreased to obtain a strictly better solution), and $\lambda_i > 0$ for all $i$ (otherwise the solution is infeasible since $\gamma > 0$).

The feasible domain is compact and the objective function is continuous and bounded above: in fact,
$\lambda_i \leq \beta$ and $\lambda_i \geq \frac{\gamma}{\beta^{r-1}}$ for all $i$  hence $\kappa^* \leq \frac{\beta^r}{\gamma}$.
Therefore, the maximum must be attained (extreme value theorem). Let $\lambda^*$ be an optimal solution.

For $r = 2$, $\lambda^*$ must satisfy
 $\lambda_1^* + \lambda_2^* = \beta$,
 $\lambda_1^* \lambda_2^* = \gamma$ and
 $\lambda_1^* \geq \lambda_2^*$, hence
$\lambda_1^* = \frac{\beta}{2} \left( 1 + \sqrt{1 - \gamma \frac{4}{\beta^2}} \right)$
and
$\lambda_2^* = \frac{\beta}{2} \left( 1 - \sqrt{1 - \gamma \frac{4}{\beta^2}} \right)$ which gives the result.

For $r \geq 3$, let us show that $\lambda^*_i = \lambda^*_{i+1}$ for all $2 \leq i \leq r-2$. Assume  $\lambda^*_i > \lambda^*_{i+1}$ for some $2 \leq i \leq r-2$. Replacing $\lambda^*_i$ and $\lambda^*_{i+1}$ by their average will keep their sum constant while strictly increasing their product hence this generates another optimal solution, a contradiction since the constraint $\prod_i \lambda_i \geq \gamma$ must be active at optimality. Therefore, the above optimization problem is equivalent to
\begin{equation} \label{eq:aproblem}
\begin{aligned}
\kappa^* \quad = \quad
\max_{x \in \mathbb{R}^3}
 \quad \quad \frac{x_1}{x_3}  \quad
  \text{ such that }  \quad
	  & x_1 + (r-2) x_2 + x_3 = \beta,  \\
	&    x_1  x_2^{r-2} x_3 = \gamma, \text{ and } \\
	&  x_1 \geq x_2 \geq x_3 \geq 0.
\end{aligned}
\end{equation}
Let us consider a relaxation of the above problem by dropping the constraint $x_1 \geq x_2 \geq x_3 \geq 0$; it will be shown that the solution of the relaxed problem satisfies the constraints automatically. 
The first-order optimality conditions of the relaxed problem are given by 
\begin{align*}
\frac{1}{x_3} & = \lambda + \mu x_2^{r-2} x_3 = \lambda + \mu \frac{\gamma}{x_1},  \\
0 & = \lambda (r-2) + \mu (r-2) x_2^{r-3} x_1 x_3 = \lambda (r-2)  + \mu (r-2) \frac{\gamma}{x_2} , \\
\frac{-x_1}{x_3^2} & = \lambda  + \mu x_2^{r-2} x_1 = \lambda + \mu \frac{\gamma}{x_3},
\end{align*}
 where $\lambda$ and $\mu$ are the Lagrangian multipliers for the first and second constraint of \eqref{eq:aproblem}, respectively. 
Multiplying the first equality by $x_1$ and the third by $x_3$ and summing them up gives
\[
x_1 + x_3 = \frac{-2 \mu \gamma}{\lambda}.
\]
Multiplying the second equality by $x_2$ gives  $x_2 = \frac{-\mu \gamma}{\lambda}$ hence $x_2 = \frac{x_1 + x_3}{2}$.
This gives $x_1 + (r-2) \frac{x_1 + x_3}{2} + x_3 = \beta$ hence $x_1 + x_3 = 2 \frac{\beta}{r}$, which combined with
 $x_1  \left(\frac{x_1 + x_3}{2}\right)^{r-2} x_3 = \gamma$ gives $x_1  x_3 = \gamma \left(\frac{r}{\beta}\right)^{r-2}$. 
 Finally, the solution is given by 
\[
x_1^* = \frac{\beta}{r} \left( 1 + \sqrt{1 - \gamma \left(\frac{r}{\beta}\right)^r} \right), 
\quad 
x_3^* = \frac{\beta}{r} \left( 1 - \sqrt{1 - \gamma \left(\frac{r}{\beta}\right)^r} \right)
\quad \text{ and } \quad x_2^* = \frac{ x_1^* + x_3^*}{2}. 
\] 
It can be seen that for $\beta \geq \gamma$ and $0 < \gamma \leq 1$, we have $x_1^* \geq x_2^* \geq x_3^* \geq 0$,
which satisfies the third constraint of \eqref{eq:aproblem} automatically.
Hence, the $(x_1^*, x_2^*, x_3^*)$ above is the optimal solution of \eqref{eq:aproblem}.

\end{document}